\documentclass[letterpaper]{article} 
\usepackage[preprint]{aaai2027}  
\usepackage[hyphens]{url}  
\usepackage{graphicx} 
\usepackage{booktabs}
\usepackage{multirow}
\usepackage{array}
\usepackage{dblfloatfix}
\usepackage{amsmath}
\usepackage{amssymb}
\usepackage{amsthm}
\usepackage{natbib}  
\usepackage{caption} 
\usepackage{algorithm}
\usepackage{algorithmic}

\usepackage{newfloat}
\usepackage{listings}
\DeclareCaptionStyle{ruled}{labelfont=normalfont,labelsep=colon,strut=off} 
\floatstyle{ruled}
\newfloat{listing}{tb}{lst}{}
\floatname{listing}{Listing}

\usepackage{booktabs}

\newtheorem{definition}{Definition}
\newtheorem{proposition}{Proposition}

\newtheorem{corollary}{Corollary}
\newcommand{\std}[1]{{\footnotesize$\pm$#1}} 
\title{Diverse and Plausible Algorithmic Recourse via Tractable Recourse Distributions}

\author{
    Anagha Sabu,
    Hrithik Suresh,
    Narayanan C. Krishnan
}

\affiliations{
    Mehta Family School of Data Science and Artificial Intelligence\\
    Indian Institute of Technology Palakkad\\
    Palakkad, Kerala, India\\
    \{142404001,142314001\}@smail.iitpkd.ac.in, ckn@iitpkd.ac.in
}

\begin{document}

\maketitle

\begin{abstract}

Algorithmic recourse seeks to help individuals reverse unfavorable automated decisions by recommending actionable changes that achieve a desired outcome. As an individual usually has several distinct routes to a favorable decision, and different people can act on different ones, a recourse system should offer multiple realistic alternatives rather than one. Existing approaches formulate recourse as an optimization problem that constructs one or a small set of counterfactuals rather than modeling the underlying space of feasible solutions, and in practice each sacrifices diversity, plausibility, or feasibility to secure the others. We propose Tractable Recourse Distributions, a probabilistic framework that represents the space of feasible alternatives for a given factual instance as a probability distribution over favorable outcomes. For commonly used cost functions based on proximity and the number of feature changes, we show that this distribution admits an exact representation as a probabilistic circuit, obtained by exponentially tilting the circuit; each individual's distribution is therefore available in closed form, without retraining the model. Sampling from these distributions naturally produces diverse and plausible recourses, while the tilting parameters provide explicit control over their proximity and sparsity. Experiments on standard algorithmic recourse benchmark datasets demonstrate that the proposed framework attains diversity, plausibility, and feasibility simultaneously, while retaining sufficient probability mass over feasible counterfactuals for rejection sampling to be practical. A visual study on MNIST illustrates how the tilt strength trades proximity against validity.

\end{abstract}

\section{Introduction}
Automated decision-making systems increasingly influence consequential
decisions such as loan approval~\cite{https://doi.org/10.1111/jofi.13090},
hiring~\cite{Raghavan_2020}, insurance~\cite{Prince2019ProxyDI}, and college
admissions~\cite{10.1609/aimag.v35i1.2504}. When such a system produces an
unfavorable decision, providing an explanation alone is often
insufficient~\cite{Wachter2018CounterfactualGDPR,barocas2020hidden}; affected
individuals also need actionable guidance on how the decision can be reversed.
Algorithmic recourse (AR) addresses this problem by identifying changes to an
individual's attributes that would lead to a desired
prediction~\cite{Wachter2018CounterfactualGDPR,Ustun2019ActionableRecourse,Karimi2022AlgorithmicRecourse}.
In practice, however, there is rarely a unique path to a favorable
(positive-class) outcome. Different interventions may suit different individuals depending on their circumstances, constraints and priorities ~\cite{Ustun2019ActionableRecourse,barocas2020hidden}.
An effective recourse system should therefore provide not merely a single recommendation, but a diverse set of realistic alternatives ~\cite{Mothilal2020DiCE}.

Existing AR approaches typically formulate recourse as an optimization problem that constructs one or a small set of counterfactuals satisfying objectives including validity, proximity, sparsity, and
plausibility~\cite{Wachter2018CounterfactualGDPR, Ustun2019ActionableRecourse,ijcai2020p395, Poyiadzi_2020,Guidotti2022Counterfactual}. Methods that generate multiple recommendations generally introduce an explicitly diversity-promotion objective or enumerate several solutions to the same optimization problem
objectives~\cite{Mothilal2020DiCE, nemecek2025generatinglikelycounterfactualsusing}.
Although this produces multiple counterfactuals, it continues to treat recourses as isolated points rather than modeling the underlying space of alternatives available to an individual.  As a result, the number of recourses must usually be specified in advance, their relative likelihood is not represented, and diversity and plausibility are handled as separate--and potentially competing--objectives.

We instead reformulate AR as the construction of an individualized probability distribution over favorable alternatives for a factual instance. We call this object a Tractable Recourse Distribution (TRD). Starting from the positive-class distribution, TRD uses exponential tilting to favor outcomes that are proximal and sparse relative to the factual instance. When the positive-class distribution is represented by a probabilistic circuit (PC) ~\cite{Choi2020ProbabilisticCircuits,Poon2011SPN,composit2021}, this transformation can be exactly realized while retaining tractable normalization, conditioning, likelihood evaluation, and sampling. Diverse recourses can therefore be generated through sampling rather than by repeatedly solving an optimization problem, with actionability enforced through conditioning and validity and causal consistency imposed during generation. Overall, we make the
following contributions:
\begin{itemize}
    \item We introduce Tractable Recourse Distributions (TRD), a distributional formulation of AR that represents the favorable alternatives to a factual instance as an individualized probability distribution, rather than constructing a fixed set of optimized feasible counterfactuals. 
    \item We show that smooth and decomposable PCs are closed under exponential tilting by additively separable recourse costs. This yields an exact construction of TRDs for commonly used proximity and sparsity objectives while preserving normalization, likelihood evaluation, conditioning, and sampling without retraining or per-individual optimization. 
    \item We characterize TRDs theoretically by showing that exponential tilting concentrates probability mass towards lower-cost interventions, including proximal and sparse counterfactuals.
    \item We demonstrate on three standard tabular AR benchmarks and MNIST that TRD generates diverse and plausible recourse sets and support joint control over proximity and sparsity. Despite enforcing classifier validity and causal consistency through rejection, TRDs retain sufficient feasible probability mass to provide recourse to every evaluated factual within a fixed sampling budget. 
\end{itemize}

\section{Related Works}
Most AR methods construct counterfactuals by searching for individual instances that attain the desired prediction while satisfying objectives such as proximity, sparsity, actionability, and plausibility ~\cite{Karimi2022AlgorithmicRecourse,Guidotti2022Counterfactual}. Gradient-based methods optimize differentiable objectives over the input space~\cite{Wachter2018CounterfactualGDPR,Mothilal2020DiCE}, whereas mixed-integer and constraint-based methods encode validity and actionability as constraints explicitly ~\cite{Ustun2019ActionableRecourse,ijcai2020p395}. Other approaches incorporate causal models ~\cite{Karimi2021CausalRecourse,Kugelgen2022FairCausalRecourse}, search over density-weighted paths ~\cite{Poyiadzi_2020} or learned data manifolds ~\cite{Pawelczyk_2020,Panagiotou_2024} or directly optimize likelihood under a tractable probabilistic model. In particular PAR performs approximate MAP inference over a Sum-Product Network to obtain plausible actionable counterfactuals \cite{sabu2026plausibility}. Despite their different mechanisms, these methods formulate recourse primarily as the construction of one counterfactual-or a fixed collection of counterfactual points-rather than an individualized distribution over possible recourses. 

A smaller body of work explicitly generates multiple recourses.  DiCE jointly optimizes a predetermined number of counterfactuals using a diversity-promoting objective ~\cite{Mothilal2020DiCE}; while solver-based methods, like LiCE~\cite{nemecek2025generatinglikelycounterfactualsusing}, obtain multiple solutions through solution pools. LiCE is especially related to our work: it incorporates likelihood under a Sum-Product Network into a mixed-integer optimization problem and enumerates near-optimal solutions, whereas its MIO variant omits the likelihood term. These approaches nevertheless construct a finite set whose size is specified in advance, with diversity imposed through an explicit objective, constraints, or repeated optimization. They therefore do not represent the relative probability or broader structure of the recourse strategies available to an individual. In contrast, TRD models an individualized distribution over favorable alternatives, from which diverse recourses are generated through sampling while plausibility is inherited from the learned positive-class distribution.

\section{Preliminaries}
\label{sec:prelims}
\subsection{Notation}
\label{sec:notation}

Let $\mathcal{X}=\mathcal{X}_1\times\cdots\times\mathcal{X}_D$
denote the input space, and let $\mathbf{x}=(x_1,\ldots,x_D)\in\mathcal{X}$ denote an instance,
where $x_j$ is its $j$-th feature. We consider a binary
classification setting with an unfavorable class $0$ and a
favorable class $1$. Let $s:\mathcal{X}\rightarrow[0,1]$
denote the classifier's favorable-class score, and let
$f(\mathbf{x})=\mathbb{I}[s(\mathbf{x})\geq\tau]$ denote its
prediction under decision threshold $\tau$. The distribution $p^{+}$ models favorable outcomes and serves as the base distribution from which individualized tractable recourse distributions are constructed.

\subsection{Algorithmic Recourse}
\label{sec:recourse}
We follow the terminology established in the AR
literature~\cite{Wachter2018CounterfactualGDPR,Karimi2022AlgorithmicRecourse}. Let $\mathbf{x}^{-}\in\mathcal{X}$ denote a factual instance receiving an unfavorable prediction, i.e., $f(\mathbf{x}^{-})=0$. A \emph{counterfactual} for $\mathbf{x}^{-}$ is an instance $\mathbf{x}^{+}\in\mathcal{X}$ that receives the desired prediction,
\( f(\mathbf{x}^{+})=1. \)

A counterfactual is \emph{feasible} if it satisfies the actionability and causal constraints associated with the factual instance. Actionability requires immutable attributes to remain unchanged and mutable attributes subject to directional constraints to change only in the prescribed direction. Let $\mathcal{I}\subseteq[D]$ denote the immutable features and let $\mathcal{M}\subseteq[D]\setminus\mathcal{I}$ denote the features subject to monotonicity constraints. We denote by $\mathcal{A}(\mathbf{x}^{-})$ the corresponding actionable set, which requires
\( x^{+}_j=x^{-}_j, \forall j\in\mathcal{I}, \) together with the prescribed directional constraints for $j\in\mathcal{M}$. Let $\mathcal{C}(\mathbf{x}^{-})$ denote the set of instances satisfying the known causal constraints. A counterfactual $\mathbf{x}^{+}$ is therefore feasible if \( \mathbf{x}^{+}\in
\mathcal{A}(\mathbf{x}^{-})\cap\mathcal{C}(\mathbf{x}^{-}).
\) 

A \emph{recourse} is a feasible counterfactual that is also proximal and sparse with respect to the factual instance. Proximity measures the magnitude of the recommended changes, whereas sparsity measures the number of features that must be modified. In this work, these properties are quantified by $\|\mathbf{x}^{+}-\mathbf{x}^{-}\|_1$ and
$\|\mathbf{x}^{+}-\mathbf{x}^{-}\|_0$, respectively, and are promoted through the recourse distribution rather than imposed as hard feasibility constraints.

Plausibility is an additional desirable property of a recourse:
it measures whether the recommended instance is realistic under
the distribution of favourable outcomes. It is not included in
the definition of feasibility or recourse.


\subsection{Probabilistic Circuits}
\label{sec:pcs}

A probabilistic circuit (PC) is a parameterized directed acyclic graph that represents a probability distribution over $\mathcal{X}$. It consists of input, product, and sum nodes. Each node $n$ has a scope $\phi(n)\subseteq[D]$, corresponding to the features represented by the subcircuit rooted at $n$, and computes a function $p_n(\mathbf{x}_{\phi(n)})$ recursively as
\[
p_n(\mathbf{x}_{\phi(n)})=
\begin{cases}
f_n(x_j),
    & \text{if $n$ is an input node},\\[2mm]
\displaystyle\prod_{c\in\operatorname{ch}(n)}
p_c(\mathbf{x}_{\phi(c)}),
    & \text{if $n$ is a product node},\\[2mm]
\displaystyle\sum_{c\in\operatorname{ch}(n)}
\theta_{n,c}\,
p_c(\mathbf{x}_{\phi(c)}),
    & \text{if $n$ is a sum node},
\end{cases}
\]
where $\operatorname{ch}(n)$ denotes the children of $n$. The weights of every sum node satisfy $\theta_{n,c}\geq 0$ and $\sum_{c\in\operatorname{ch}(n)}\theta_{n,c}=1$. The distribution represented by the PC is computed at its root.

We consider smooth and decomposable PCs. Smoothness requires the children of every sum node to have identical scopes, whereas decomposability requires the children of every product node to have pairwise disjoint scopes. Assuming normalized input distributions, these properties ensure that every subcircuit represents a normalized distribution over its scope and enable exact likelihood evaluation, marginalization, conditioning, and sampling. We use a PC to represent the positive-class distribution $p^{+}$.

\subsection{Exponential Tilting} 
Let $p$ be a probability distribution over $\mathcal{X}$ and let $G:\mathcal{X}\rightarrow\mathbb{R}$ be a measurable cost function such that \( 0<Z_G  = \int_{\mathcal{X}}p(\mathbf{x}) \exp\{-G(\mathbf{x})\}\,d\mathbf{x} <\infty. \) The exponential tilt of $p$ by $G$ is the distribution~\cite{siegmund1976importance,li2023tilted}
\[
\widetilde{p}_G(\mathbf{x})
=
\frac{
p(\mathbf{x})\exp\{-G(\mathbf{x})\}
}{
Z_G
}.
\]
Exponential tilting reweights the base distribution toward regions of lower cost while preserving its support whenever $G$ is finite. The strength of the transformation is controlled through the parameters appearing in $G$: increasing their values places progressively greater probability mass on lower-cost instances. In our setting, $G$ is factual-specific and combines separately weighted proximity and sparsity costs.

\section{Methodology}
\label{sec:methodology}

\subsection{Tractable Recourse Distributions}
Let $\mathbf{x}^{-}$ be a factual instance receiving an unfavorable prediction. Rather than directly constructing a finite set of counterfactuals, we seek to define an individualized probability distribution over favorable outcomes for $\mathbf{x}^{-}$. Such a distribution should retain the plausibility encoded by the positive-class distribution while assigning greater probability to outcomes that are proximal and sparse relative to the factual instance. Feasibility is handled subsequently: actionability constraints are imposed through conditioning, whereas causal constraints are enforced during recourse generation through rejection sampling.

The positive-class distribution $p^{+}$ already represents the  distribution of plausible favorable outcomes. We therefore personalize $p^{+}$ to the factual instance rather than learning a separate model for each individual from scratch. Specifically, we define the factual-specific intervention cost
\[
G_{\delta,\nu}(\mathbf{x};\mathbf{x}^{-})
=
\delta\|\mathbf{x}-\mathbf{x}^{-}\|_{1}
+
\nu\|\mathbf{x}-\mathbf{x}^{-}\|_{0},
\]
where $\delta,\nu\geq 0$ control the emphasis on proximity and sparsity, respectively.

\begin{definition}[Tractable Recourse Distribution]
Given a factual instance $\mathbf{x}^{-}$, its tractable recourse distribution is the exponentially tilted positive-class distribution
\[
\widetilde{p}_{\delta,\nu}
(\mathbf{x}\mid\mathbf{x}^{-})
=
\frac{
p^{+}(\mathbf{x})
\exp\!\left\{
-G_{\delta,\nu}(\mathbf{x};\mathbf{x}^{-})
\right\}
}{
Z(\delta,\nu;\mathbf{x}^{-})
},
\]
where
\(
Z(\delta,\nu;\mathbf{x}^{-})
=
\int_{\mathcal{X}}
p^{+}(\mathbf{x})
\exp\!\left\{
-G_{\delta,\nu}(\mathbf{x};\mathbf{x}^{-})
\right\}
d\mathbf{x}
\)
is the normalization constant.
\label{def:trd}
\end{definition}

The base distribution promotes plausibility, while exponential tilting concentrates probability mass on favorable outcomes requiring smaller and fewer changes from the factual instance. The parameters $\delta$ and $\nu$ therefore control the relative emphasis on proximity and sparsity. The distribution defined above is subsequently conditioned on the actionable set and filtered for causal consistency to obtain feasible counterfactuals from which recourses are selected.

\paragraph{Deployment setting.}
We consider the setting in which individual training examples are unavailable at recourse-generation time. Once the classifier and the positive-class distribution $p^{+}$ have been learned, constructing a factual-specific recourse distribution requires access only to the trained classifier, the learned PC representing $p^{+}$, and the factual instance $\mathbf{x}^{-}$. There are no training examples available during deployment.

\subsection{Exact Construction using Probabilistic Circuits}

Definition~\ref{def:trd} provides a principled formulation of individualized recourse distributions. However, its exact normalization and sampling are generally intractable for most unrestricted deep generative models. Its exact realization using a PC follows from the additive structure of the intervention cost. 



\begin{definition}[Additive Separability]
\label{def:additive_separability}
A factual-specific cost $G:\mathcal{X}\times\mathcal{X}\rightarrow\mathbb{R}$ is additively separable if
\(
G(\mathbf{x};\mathbf{x}^{-})
=
\sum_{j=1}^{D} g_j(x_j;x_j^{-}),
\)
where each $g_j$ depends only on feature $j$.
\end{definition}

The proximity and sparsity cost defining TRD is additively separable, since \( g_j(x_j;x_j^{-}) = \delta |x_j-x_j^{-}| + \nu \mathbb{I}[x_j\neq x_j^{-}].\) Thus,
\(
\exp\{-G(\mathbf{x};\mathbf{x}^{-})\}
=
\prod_{j=1}^{D}
\exp\{-g_j(x_j;x_j^{-})\},
\)
which aligns with the recursive structure of a smooth and
decomposable PC.

For a node $n$, define the cost restricted to its scope as
\(
G_{\phi(n)}(\mathbf{x}_{\phi(n)};\mathbf{x}^{-}_{\phi(n)})
=
\sum_{j\in\phi(n)}g_j(x_j;x_j^{-}).
\)

\begin{proposition}[Exact realization of TRD]
Let $p^{+}$ be represented by a smooth and decomposable PC, and let $G$ be additively separable.
Suppose that, for every input node, its exponentially tilted
distribution and corresponding normalizer can be computed
exactly. Then the TRD admits an exact realization as a smooth
and decomposable PC with the same topology as $p^{+}$.

For every node $n$, the transformed subcircuit
computes
\[
\widetilde{p}_n(\mathbf{x}_{\phi(n)})
=
\frac{
p_n(\mathbf{x}_{\phi(n)})
\exp\{-G_{\phi(n)}(\mathbf{x}_{\phi(n)};
                  \mathbf{x}^{-}_{\phi(n)})\}
}{
Z_n
},
\]
where
\(
Z_n
=
\int
p_n(\mathbf{x}_{\phi(n)})
\exp\{-G_{\phi(n)}(\mathbf{x}_{\phi(n)};
                  \mathbf{x}^{-}_{\phi(n)})\}
\,d\mathbf{x}_{\phi(n)}.
\)
The transformed circuit is constructed recursively as follows.

For an input node $n$ with $\phi(n)=\{j\}$,
\[
\widetilde{p}_n(x_j)
=
\frac{
p_n(x_j)\exp\{-g_j(x_j;x_j^{-})\}
}{
Z_n
}.
\]

For a product node $n$,
\[
Z_n=\prod_{c\in\operatorname{ch}(n)}Z_c,
\qquad 
\widetilde{p}_n
=
\prod_{c\in\operatorname{ch}(n)}\widetilde{p}_c.
\]

For a sum node $n$, define
\[
Z_n
=
\sum_{c\in\operatorname{ch}(n)}
\theta_{n,c}Z_c,
\qquad
\widetilde{\theta}_{n,c}
=
\frac{\theta_{n,c}Z_c}{Z_n}.
\]
Then
\(
\widetilde{p}_n
=
\sum_{c\in\operatorname{ch}(n)}
\widetilde{\theta}_{n,c}\widetilde{p}_c.
\)
At the root $r$, $\widetilde{p}_r$ is exactly the tractable
recourse distribution
$\widetilde{p}_{\delta,\nu}
(\mathbf{x}\mid\mathbf{x}^{-})$.
\end{proposition}

\begin{proof}
    See Supplementary Material Section~A.1.
\end{proof}
\begin{corollary}[Closure under Additive Exponential Tilting]
\label{cor:closure}
Subject to tractable closure of the input distributions, smooth and decomposable PCs are closed under exponential tilting by additively separable costs. The resulting circuit retains exact likelihood evaluation, marginalization, conditioning, and sampling, without retraining the base model.
\end{corollary}

For the input distributions (categorical and histogram) popularly used in the PC literature \cite{arya2024neuralnetworkapproximatorsmarginal}, the tilted probabilities and leaf normalizers are obtained exactly by element-wise reweighting over their finite support.

\subsection{Theoretical Property}
We now characterize how exponential tilting concentrates probability mass toward lower-cost interventions. 
\begin{proposition}[Concentration of Intervention Cost]
Let
$\widetilde{p}(\mathbf{x}\mid\mathbf{x}^{-})
=p^{+}(\mathbf{x})
\exp\{-G(\mathbf{x};\mathbf{x}^{-})\}/Z_G(\mathbf{x}^{-})
$
where $G(\mathbf{x};\mathbf{x}^{-})\geq 0$. Then, for every
$\epsilon>0$,
\[
\Pr_{\mathbf{x}\sim\widetilde{p}}
\left[
G(\mathbf{x};\mathbf{x}^{-})>\epsilon
\right]
\leq
\frac{\exp(-\epsilon)}
{Z_G(\mathbf{x}^{-})}.
\]
\label{prop:concentration}
\end{proposition}
\begin{proof}
    See Supplementary Material Section~A.2.
\end{proof}
\noindent
For the TRD cost \( G_{\delta,\nu}(\mathbf{x};\mathbf{x}^{-}) \), Proposition \ref{prop:concentration} implies exponential suppression of counterfactuals that are distant from the factual or modify many features. Thus, $\delta$ and $\nu$ control the emphasis placed on proximity and sparsity, respectively.

\subsection{Feasibility Enforcement}

Proximity and sparsity are incorporated into TRD as soft
preferences. Actionability, in contrast, imposes hard restrictions on the permissible changes: immutable features must remain fixed, and features subject to monotonicity constraints may change only in the prescribed direction.

Given the actionable set $\mathcal{A}(\mathbf{x}^{-})\subseteq\mathcal{X}$, we obtain the actionable TRD by conditioning:
\[
\widetilde{p}_{\mathcal A}
(\mathbf{x}\mid\mathbf{x}^{-})
=
\frac{
\widetilde{p}_{\delta,\nu}
(\mathbf{x}\mid\mathbf{x}^{-})
\mathbb{I}
[\mathbf{x}\in\mathcal{A}(\mathbf{x}^{-})]
}{
\Pr_{\widetilde{p}_{\delta,\nu}}
\!\left(
\mathbf{X}\in\mathcal{A}(\mathbf{x}^{-})
\mid\mathbf{x}^{-}
\right)
}.
\]
Conditioning removes probability mass from inadmissible
interventions and preserves the relative probabilities assigned
by TRD to instances within the actionable set.

The actionability constraints $\mathcal{A}(\mathbf{x}^{-})$ decomposes into feature-wise restrictions. Immutability is enforced by restricting the corresponding input distribution to $x_j=x_j^{-}$, while monotonicity is enforced by restricting its support to values satisfying the prescribed direction of change. Because PCs support exact conditioning on such input-level events, the actionable TRD is represented exactly as a PC and retains tractable likelihood evaluation and sampling.

Actionability alone does not ensure that a sampled instance is a counterfactual or that it satisfies the causal constraints. Classifier validity depends on the external predictive model, whereas causal constraints may involve multiple features and cannot generally be encoded as independent input-node restrictions. We therefore sample from the actionable TRD and reject instances that do not receive the desired prediction or violate the specified causal constraints. The accepted instances are feasible counterfactuals; their proximity and sparsity are promoted by the underlying tilt.

\subsection{Diverse Recourse Set Construction}
Unlike point-wise approaches that construct one or a predetermined number of recourses, the TRD represents a distribution from which an arbitrary number of candidate interventions can be sampled. Diversity therefore, does not require an additional diversity-promoting objective: independent samples can naturally arise from different regions of the individualized recourse distribution.

After enforcing feasibility as described in the preceding subsection, we obtain a pool of feasible counterfactuals. The proximity and sparsity tilts bias this pool toward interventions that require smaller and fewer changes from the factual instance. Directly presenting the entire pool, however would be redundant and impractical, as several samples may correspond to similar recommendations. 

We therefore construct a compact and representative recourse set by clustering the feasible counterfactuals using $k$-medoids returning the medoid of each cluster. The medoids are themselves sampled feasible counterfactuals, ensuring that every returned recommendation satisfies the required constraints. The number of clusters $k$ is selected separately for each factual instance using the silhouette score~\cite{ROUSSEEUW198753}. Therefore, the number of returned recourses adapts to the structure of the sampled distribution rather than being fixed uniformly across individuals.

Importantly, clustering serves only to summarize the diversity inherent in the feasible sample pool; it does not introduce diversity into TRD. We examine this distinction empirically by comparing random selection, fixed $k$ clustering, and adaptive clustering.

Let \(|\mathcal{C}|\) denote the size of the probabilistic circuit, \(L\) the number of input-distribution nodes, and \(B\) the sampling budget. Constructing an individualized TRD requires \(O(|\mathcal{C}|)\) time, while imposing feature-wise actionability constraints requires \(O(L)\) time. Rejection sampling has complexity \(O(BL)\), excluding the costs of classifier evaluation and causal-consistency checking.

\section{Experimental Setup}

\paragraph{Datasets.}
We adopt the LiCE benchmark~\cite{nemecek2025generatinglikelycounterfactualsusing}, including its five-fold data splits, preprocessing,
trained classifiers, dataset-specific actionability and causal constraints, and definitions of distance and sparsity. The benchmark comprises Adult~\cite{BeckerKohavi1996Adult}, German
Credit~\cite{Hofmann1994GermanCredit}, and Give Me Some Credit
(GMSC)~\cite{KaggleGiveMeSomeCredit}. For each dataset and fold, we learn the positive-class distribution ($p^+$) exclusively from ground-truth favorable instances in the corresponding training split. The PC structure and parameters are learned with LearnSPN \cite{pmlr-v28-gens13}  as implemented in SPFlow~\cite{Molina2019SPFlow}. 

The tabular benchmarks constitute our primary evaluation setting. We additionally use MNIST~\cite{726791} to examine the behavior of TRD in a higher-dimensional domain. We evaluate five source-to-target digit tasks and visualize representative results for $8 \rightarrow 0$ and $7 \rightarrow 1$. Each $28 \times 28$ image is binarized using an intensity threshold of $0.5$. For each target digit, we learn a class-conditional PC from its training images, representing each pixel with a two-bin histogram. Given a factual image from the source class, we tilt the target-class PC toward the factual using pixelwise Hamming distance. We evaluate the resulting counterfactuals in terms of validity, plausibility, and diversity.

\paragraph{Methods for Comparison.}
On the tabular benchmarks, we compare TRD with DiCE, LiCE, and MIO (LiCE formulation with its likelihood objective removed), which are the methods in the LiCE benchmark capable of generating multiple counterfactuals. All methods are evaluated using the same data splits, trained classifiers, preprocessing, and dataset-specific feasibility constraints. Each baseline is configured to return at most ten counterfactuals for every factual instance. TRD instead returns an adaptively sized recourse set using the sampling and clustering procedure. Baseline comparisons are restricted to the tabular datasets; the MNIST experiments examine the behavior of exponential tilting in a higher-dimensional setting. The sampling budget for TRD is set at 10,000. TRD requires tuning of the hyper-parameters $\delta$ and $\nu$. The optimal values are obtained through a grid search using a separate training subset for every fold. Further details are provided in  Section~B.

\subsection{Evaluation Metrics}
\paragraph{Validity, Feasibility, and Coverage.}
A generated candidate is a counterfactual if the classifier assigns
it a favorable-class score above the dataset-specific decision
threshold: $0.5$. We report \emph{Valid \%}, the proportion of generated candidates satisfying this criterion.

We evaluate the components of feasibility separately.
\emph{Actionable \%} is the proportion of candidates satisfying
the immutability and monotonicity constraints, while
\emph{Causal \%} is the proportion satisfying the available causal
constraints. These percentages are computed independently over all
generated candidates. A candidate is feasible only if it is valid,
actionable, and causally consistent. \emph{Served \%} denotes the
proportion of factual instances for which at least one feasible
counterfactual is obtained, and \emph{\# ret.} is the mean number
of feasible counterfactuals returned per served factual and averaged over all factuals.

\paragraph{Plausibility, proximity, and sparsity.}
Plausibility is measured by negative log-likelihood under the
positive-class PC $p^{+}$; lower values indicate more plausible
instances. As any member of a returned set may be acted upon,
we report the best, mean, and worst NLL within each final recourse
set. Proximity is measured using the benchmark's MAD-weighted
$\ell_1$ distance~\cite{Wachter2018CounterfactualGDPR} from the factual instance, while sparsity is the
number of modified features. Unless stated otherwise, these metrics
are computed over the final feasible recourse set returned for each
factual and averaged over all factuals.

\paragraph{Diversity.}
We measure diversity using Count-Diversity \cite{Mothilal2020DiCE}, the mean pairwise
fraction of features on which two returned recourses differ, and
$k_{\mathrm{eff}}$, the number of distinct sets of modified
features represented in the returned set. Whereas Count-Diversity
measures feature-level separation, $k_{\mathrm{eff}}$ captures the
number of distinct intervention strategies.


For MNIST, validity is the proportion of generated images classified
as the target digit, plausibility is measured by NLL under the
target-class PC, sparsity is the number of modified pixels ($\ell_0 =\ell_1$), and
Count-Diversity is computed over pixels.

Further details on Supplementary Section~F.

\section{Results and Discussion}
\paragraph{Q1. Does TRD provide a better overall balance of
recourse quality than pointwise optimization methods?}
Tables~\ref{tab:diversity} and~\ref{tab:constraints} compare TRD with existing diverse recourse methods in terms of diversity, plausibility, constraint satisfaction, and coverage. TRD does not maximize every individual metric; rather, it provides the most consistent balance across them. It generates multiple distinct recourse strategies, maintains high plausibility across the returned set, and serves every evaluated factual instance with feasible counterfactuals. 

The competing methods exhibit substantially sharper trade-offs. DiCE achieves the highest Count-Diversity on all three datasets, but many of its generated candidates violate the feasibility requirements, resulting in low coverage on German Credit. LiCE generates highly plausible counterfactuals, but its returned sets largely collapse to a single effective recourse strategy despite containing multiple solutions. Removing the likelihood objective, as in MIO, increases the number of distinct strategies but substantially reduces plausibility. TRD avoids these extremes: although it is not always the most diverse method under Count-Diversity, it returns multiple distinct strategies that remain plausible and feasible, while serving all evaluated factual instances.

TRD samples are actionable by construction. Classifier validity and causal consistency are enforced through rejection sampling, with mean  acceptance rates of $18.5\%$, $18.1\%$, and $71.1\%$ on Adult, German Credit, and GMSC, respectively. These correspond to approximately $5.4$, $5.5$, and $1.4$ candidate draws per accepted feasible counterfactual. All evaluated factuals obtain a feasible recourse within the fixed sampling budget, and every returned recommendation is therefore valid, actionable, and causally consistent. Ablations on the sampling budget are presented in the Supplementary Material Section~C.

\begin{table}[!t]
\centering
\small
\setlength{\tabcolsep}{2pt}
\begin{tabular}{@{}l cc ccc c@{}}
\toprule
& & & \multicolumn{3}{c}{NLL\,$\downarrow$} & \\
\cmidrule(lr){4-6}
Method & Count-Div\,$\uparrow$ & $k_{\mathrm{eff}}\!\uparrow$/\#ret. & best & mean & worst  \\
\midrule
\multicolumn{7}{c}{\textit{Adult}}\\
\midrule
DiCE          & \textbf{.62}\std{.01} & \textbf{5.80}/6.40 & 20.49 & 26.29 & 31.54  \\
LiCE          & .00 & 1.00/10.00 & \textit{12.11} & \textbf{12.11} & \textbf{12.11}  \\
MIO           & .19\std{.00} & 2.45/9.68 & 17.21 & 19.41 & 22.20   \\
\textbf{TRD}  & \textit{.27}\std{.01} & \textit{2.52}/4.09 & \textbf{11.75} & \textit{14.24} & \textit{17.39}  \\
\midrule
\multicolumn{7}{c}{\textit{German Credit}}\\
\midrule
DiCE$^{\ddagger}$ & \textbf{.39} & \textit{5.00}/5.00 & 27.67 & 36.91 & 66.03  \\
LiCE          & .01\std{.01} & 1.18/10.00 & \textit{22.27} & \textit{22.35} & \textit{22.48}  \\
MIO           & .12\std{.00} & \textbf{7.72}/10.00 & 28.84 & 32.97 & 36.58  \\
\textbf{TRD}  & \textit{.21}\std{.01} & 3.36/3.49 & \textbf{17.83} & \textbf{19.49} & \textbf{21.09}  \\
\midrule
\multicolumn{7}{c}{\textit{GMSC}}\\
\midrule
DiCE          & \textbf{.51}\std{.09} & \textit{4.75}/6.38 & \textit{7.79} & 12.83 & 17.48 \\
LiCE          & .00 & 1.00/10.00 & 10.86 & \textit{10.86} & \textbf{10.86} \\
MIO           & .23\std{.03} & \textbf{8.62}/10.00 & 17.48 & 18.79 & 20.02 \\
\textbf{TRD}  & \textit{.44} & 1.83/2.67 & \textbf{7.49} & \textbf{9.74} & \textit{11.80} \\
\bottomrule
\end{tabular}
\caption{Diversity and plausibility of returned recourse sets;
mean\,$\pm$\,std over folds, on each method's feasible subset (cf.\
Table~\ref{tab:constraints}). \textbf{Bold}/\textit{italic}: best/second-best
per column within each dataset. $k_{\mathrm{eff}}$/\#ret.\ is distinct
strategies over mean number returned; NLL columns are best/mean/worst within
each set. $^{\ddagger}$DiCE serves one individual on Credit, so
its values admit no fold-level variance.}
\label{tab:diversity}
\end{table}

\begin{table}[!ht]
\centering
\small
\setlength{\tabcolsep}{3pt}
\begin{tabular}{@{}l cccc@{}}
\toprule
Method & Valid\,\%\,$\uparrow$ & Action.\,\%\,$\uparrow$ & Causal\,\%\,$\uparrow$
& Served\,\%\,$\uparrow$ \\
\midrule
\multicolumn{5}{c}{\textit{Adult}}\\
\midrule
DiCE          & 96.2\std{1.8} & \textit{43.1}\std{2.9} & \textit{72.3}\std{3.9} & 59.8\std{5.2} \\
LiCE          & 97.3\std{1.9} & \textbf{100} & \textbf{100} & 89.1\std{3.0} \\
MIO   & 93.9\std{2.2} & \textbf{100} & \textbf{100} & 97.0\std{1.4} \\
\textbf{TRD}  & \textbf{100} & \textbf{100} & \textbf{100} & \textbf{100} \\
\midrule
\multicolumn{5}{c}{\textit{German Credit}}\\
\midrule
DiCE          & \textbf{100} & \textit{0.4}\std{1.0} & \textit{80.3}\std{3.9} & 0.6\std{1.4} \\
LiCE          & \textbf{100} & \textbf{100} & \textbf{100} & \textit{99.1}\std{1.9} \\
MIO  & \textbf{100} & \textbf{100} & \textbf{100} & \textit{99.1}\std{1.9} \\
\textbf{TRD}  & \textbf{100} & \textbf{100} & \textbf{100} & \textbf{100} \\
\midrule
\multicolumn{5}{c}{\textit{GMSC}}\\
\midrule
DiCE          & \textbf{100} & \textit{63.8}\std{6.3} & --- & \textbf{100} \\
LiCE          & \textbf{100} & \textbf{100} & --- & \textbf{100} \\
MIO           & \textbf{100} & \textbf{100} & --- & \textbf{100} \\
\textbf{TRD}  & \textbf{100} & \textbf{100} & --- & \textbf{100} \\
\bottomrule
\end{tabular}
\caption{Constraint satisfaction and coverage; mean\,$\pm$\,std over
folds. The first three columns are \emph{independent} satisfaction rates over
all returned counterfactuals; a counterfactual may be valid yet violate
actionability.}
\label{tab:constraints}
\end{table}
\paragraph{Q2. Does TRD maintain plausibility across the entire returned recourse set?} 
Evaluating only the most plausible counterfactual can obscure the quality of a diverse recourse set, as an individual may ultimately select any of the returned recommendations. We therefore report the best, mean, and worst negative log-likelihood (NLL) within each returned set under the positive-class distribution $p^{+}$ in Table \ref{tab:diversity}.


TRD maintains strong plausibility throughout its returned sets. It achieves the lowest best and mean NLL on all three datasets and the lowest worst NLL on German Credit. On Adult and GMSC, LiCE attains a lower worst NLL; however, LiCE's ten returned counterfactuals largely correspond to a single effective strategy and consequently have nearly identical likelihoods. In contrast, TRD retains low NLL while returning multiple distinct strategies. Diversity-oriented methods such as DiCE and MIO generally exhibit substantially higher mean and worst NLL, indicating that some of their returned recommendations lie in considerably less plausible regions. These results show that TRD's diversity is not obtained by appending implausible counterfactuals to an otherwise plausible set.
\begin{figure}[h]
    \centering
    \includegraphics[width=\linewidth]{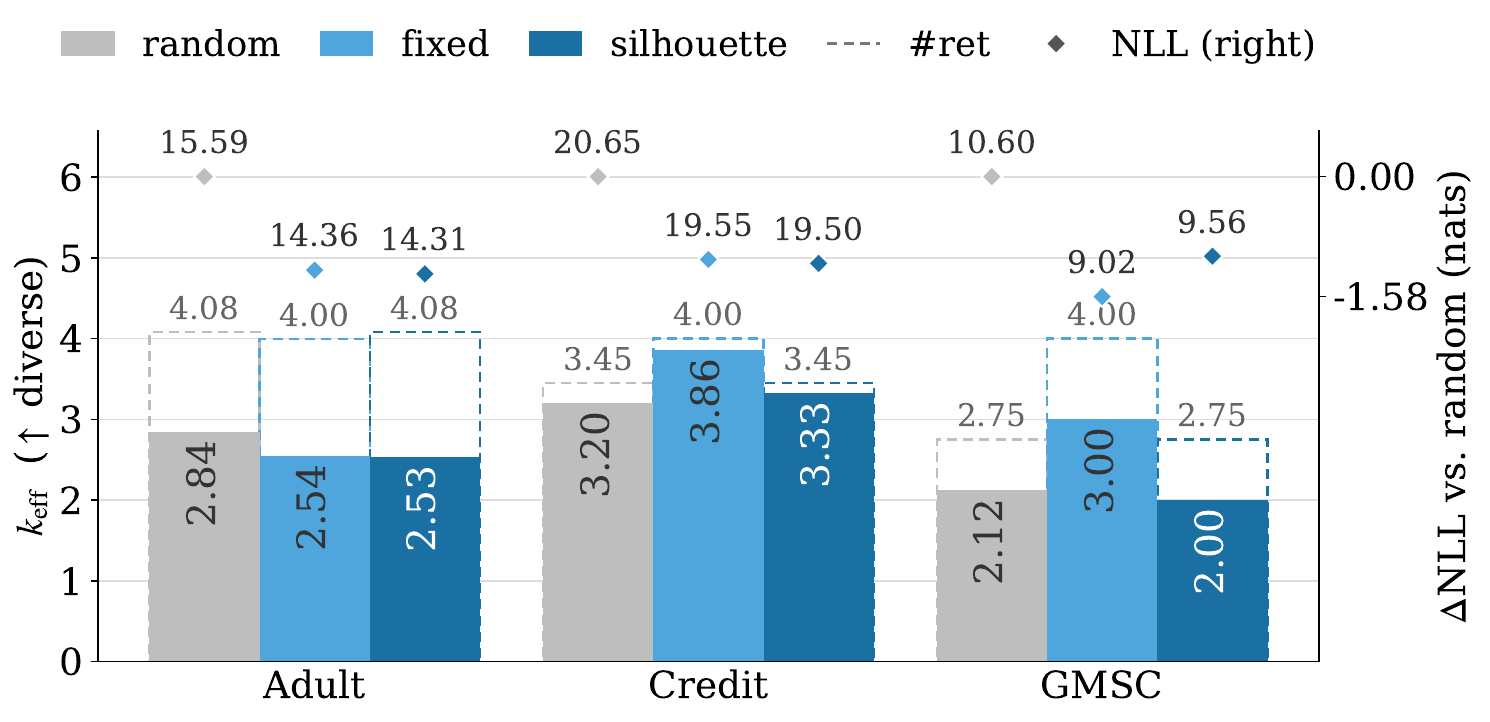}
    \caption{\textbf{Diversity is intrinsic to the recourse distribution.}
Bar vs.\ its dashed outline: distinct recourses vs.\ number returned, so the
gap is redundancy. Diamonds (right axis): mean NLL under $p^{+}$ minus
random's, so the random diamond sits at $0$ and lower is better; labels are
raw NLL. Random selection alone is already diverse; medoids add plausibility,
not diversity.}
\label{fig:q3}
\end{figure}

\paragraph{Q3. Is diversity intrinsic to TRD, or introduced by clustering?} Figure \ref{fig:q3} compares three sampling strategies: (i) randomly selecting a fixed number of recourses from the distribution, (ii) clustering with a fixed number of clusters, and (iii) adaptive clustering using the silhouette index. For each strategy we also compute the change in the average NLL of the recourse set with respect to the random baseline. The results show that even randomly sampled recourses exhibit substantial diversity. At the same time, there is only a consistent reduction in the average NLL due to clustering. These results indicate that clustering primarily summarizes the diversity already present in TRD and improves the representativeness of the returned set, rather than creating diversity through an explicit objective.

\begin{figure}[!t]
    \centering
    \includegraphics[width=1.0\linewidth]{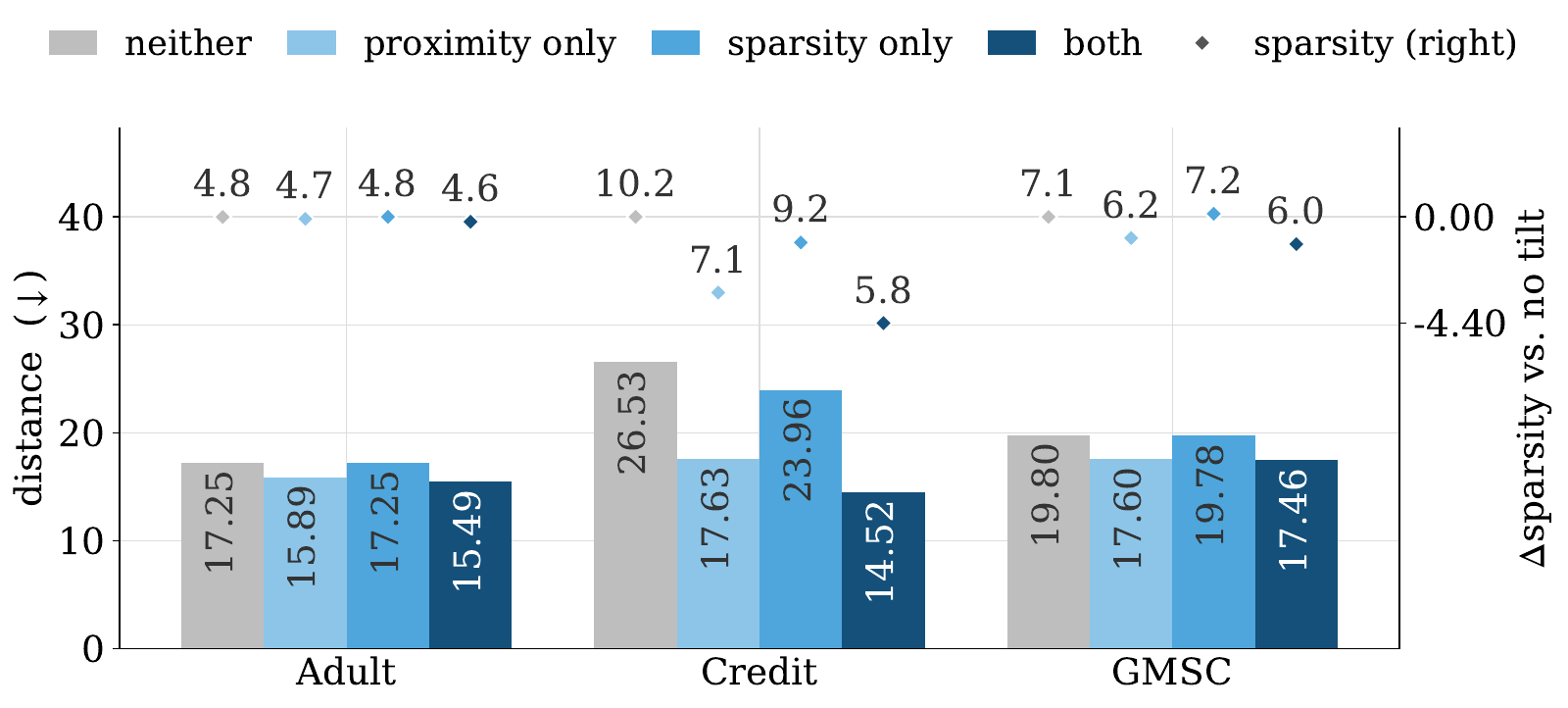}
    \caption{\textbf{Proximity and sparsity improve together under joint tilting.}
Bars (left axis): mean distance from the factual. Diamonds (right axis): mean
sparsity minus the untilted run's, so \emph{neither} sits at $0$; labels are
the raw number of features changed.}
    \label{fig:q4}
\end{figure}

\paragraph{Q4. Do the proximity and sparsity tilts provide complementary control over the recourse quality?} Figure ~\ref{fig:q4} compares the un-tilted positive-class distribution with proximity-only, sparsity-only, and joint tilting. Across all three tabular datasets, applying both tilt terms yields the lowest distance from the factual and fewest number of feature changes, indicating that proximity and sparsity can be controlled jointly through TRD. At the selected $\nu$, the sparsity tilt alone has little or no effect on Adult and GMSC; the gain there comes from the proximity term, with sparsity contributing only in combination. This complementarity is an empirical property of the evaluated datasets rather than a consequence of Proposition \ref{prop:concentration}. Detailed sensitivity analyses of the individual proximity and sparsity tilting parameters are provided in the Supplementary Material Section~D.

\paragraph{Q5. Does TRD extend beyond tabular data?}Figure~\ref{fig:tilt_sweep} visualizes representative \(8\!\rightarrow\!0\) and \(7\!\rightarrow\!1\) transformations. Without tilting, samples resemble generic instances of the target class; as the tilt strength increases, they retain progressively more pixels from the factual image, reducing the number of required changes. Excessive tilting, however, concentrates the distribution too strongly around the source image and lowers target-class validity, illustrating the trade-off between proximity and validity. 

Across the two evaluated tasks, TRD serves every factual instance with valid recourses while maintaining diversity and plausibility. These results show that the proposed construction extends computationally beyond tabular data and that its tilting mechanism behaves consistently in a high-dimensional domain, although the experiment is intended as a mechanism study rather than a claim of state-of-the-art image recourse generation.

\begin{figure}[!ht]
    \centering
    \includegraphics[width=0.9\linewidth]{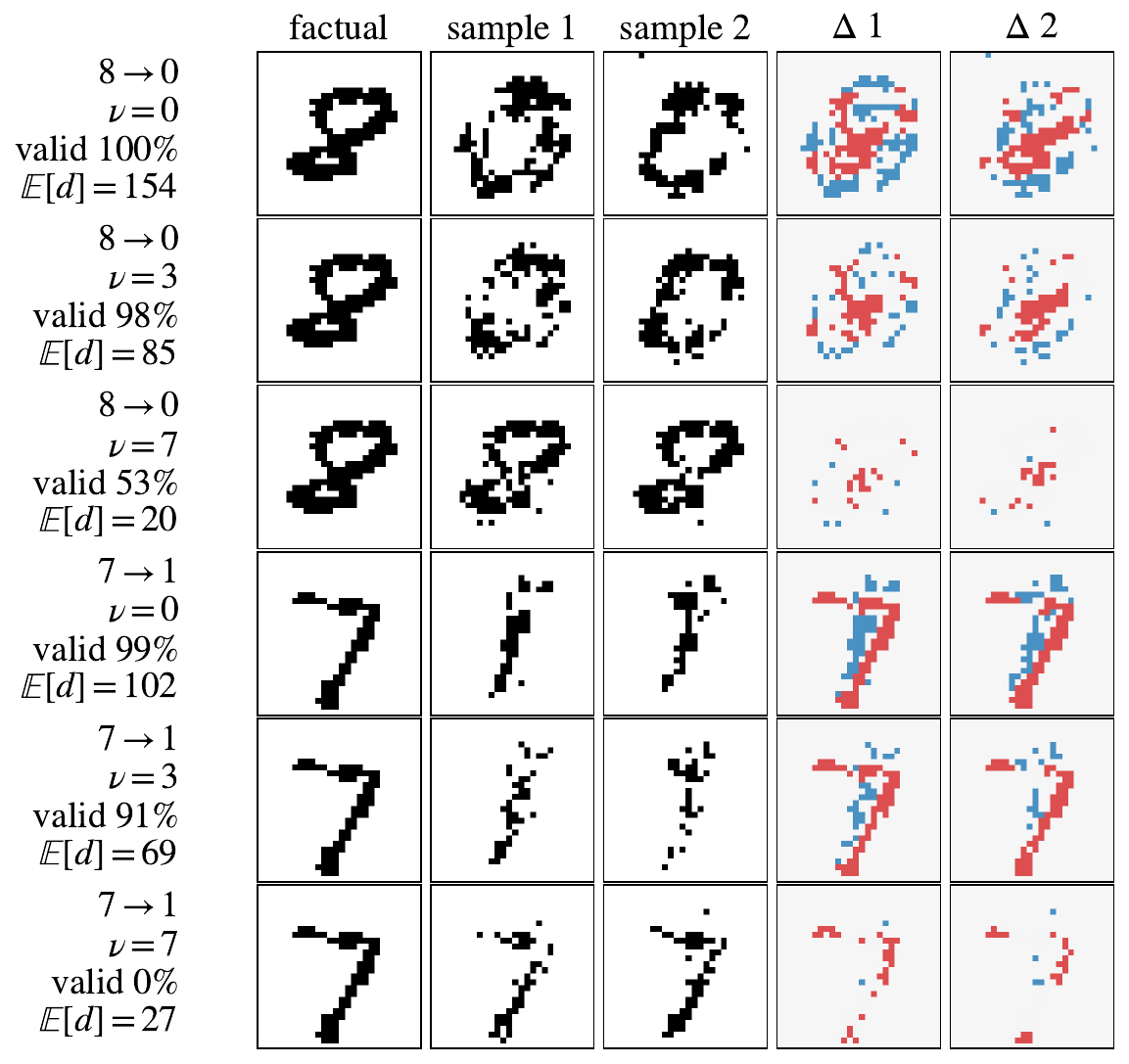}
    \caption{The tilt on MNIST. Each row tilts a target-digit circuit toward a factual source-digit image (left) at strength $\nu$; \emph{sample 1 - 2} are independent draws from the tilted circuit and $\Delta 1$--$2$ their pixel differences from the factual (red: present in the factual and removed by the sample; blue: added by the sample). }
    \label{fig:tilt_sweep}
\end{figure}

\begin{table}[!ht]
\centering\small
\setlength{\tabcolsep}{2pt}
\begin{tabular}{@{}l c ccc cc@{}}
\toprule
& & \multicolumn{3}{c}{NLL\,$\downarrow$} & & \\
\cmidrule(lr){3-5}
Task & Valid\,\%\,$\uparrow$ & best & mean & worst
& Spars.\,(px)\,$\downarrow$ & Count-Div\,$\uparrow$ \\
\midrule
$8\rightarrow0$ & 95.6 & 134.6 & 188.7 & 250.2 & 90.8 & 0.124 \\
$7\rightarrow1$ & 88.5 & 51.1 & 91.1 & 146.3 & 71.1 & 0.061 \\
\bottomrule
\end{tabular}
\caption{MNIST recourse quality at the operating point $\nu{=}3$
(all factuals served).}
\label{tab:mnist}
\end{table}

\section{Conclusion}
We introduce Tractable Recourse Distributions (TRDs), a distributional formulation of algorithmic recourse that represents favorable alternatives available to an individual rather than constructing a fixed set of optimized counterfactuals. We show that smooth and decomposable probabilistic circuits are closed under exponential tilting by additive recourse costs, enabling exact TRD construction for proximity and sparsity while preserving tractable probabilistic inference. The resulting distribution requires neither retraining nor per-instance optimization and can be constructed without access to the training data.

Our experiments show that TRDs retain sufficient probability mass over feasible counterfactuals for rejection sampling to provide a practical alternative to repeated pointwise optimization in the evaluated settings. The returned recourses are diverse and plausible, while the tabular and MNIST studies demonstrate joint control over proximity and sparsity. A remaining limitation is that the effectiveness of rejection sampling depends on the feasible probability mass, which may diminish under more restrictive validity or causal constraints. Future work will investigate compiling these constraints directly into the circuit and adapting tilt strengths to individual factuals.
\bibliography{aaai2027}
\appendix
\appendix
\onecolumn

\noindent \hrulefill
\section*{Supplementary Material}

\noindent \hrulefill
\vspace{1em}

\appendix
\section{Proofs}                                   \label{app:proofs}
  \subsection{Proof of Proposition 1}  \label{app:proof-exact}

\begin{proposition}[Exact realization of TRD]
Let $p^{+}$ be represented by a smooth and decomposable PC, and let $G$ be additively separable.
Suppose that, for every input node, its exponentially tilted
distribution and corresponding normalizer can be computed
exactly. Then the TRD admits an exact realization as a smooth
and decomposable PC with the same topology as $p^{+}$.

For every node $n$, the transformed subcircuit
computes
\[
\widetilde{p}_n(\mathbf{x}_{\phi(n)})
=
\frac{
p_n(\mathbf{x}_{\phi(n)})
\exp\{-G_{\phi(n)}(\mathbf{x}_{\phi(n)};
                  \mathbf{x}^{-}_{\phi(n)})\}
}{
Z_n
},
\]
where
\(
Z_n
=
\int
p_n(\mathbf{x}_{\phi(n)})
\exp\{-G_{\phi(n)}(\mathbf{x}_{\phi(n)};
                  \mathbf{x}^{-}_{\phi(n)})\}
\,d\mathbf{x}_{\phi(n)}.
\)
The transformed circuit is constructed recursively as follows.

For an input node $n$ with $\phi(n)=\{j\}$,
\[
\widetilde{p}_n(x_j)
=
\frac{
p_n(x_j)\exp\{-g_j(x_j;x_j^{-})\}
}{
Z_n
}.
\]

For a product node $n$,
\[
Z_n=\prod_{c\in\operatorname{ch}(n)}Z_c,
\qquad 
\widetilde{p}_n
=
\prod_{c\in\operatorname{ch}(n)}\widetilde{p}_c.
\]

For a sum node $n$, define
\[
Z_n
=
\sum_{c\in\operatorname{ch}(n)}
\theta_{n,c}Z_c,
\qquad
\widetilde{\theta}_{n,c}
=
\frac{\theta_{n,c}Z_c}{Z_n}.
\]
Then
\(
\widetilde{p}_n
=
\sum_{c\in\operatorname{ch}(n)}
\widetilde{\theta}_{n,c}\widetilde{p}_c.
\)
At the root $r$, $\widetilde{p}_r$ is exactly the tractable
recourse distribution
$\widetilde{p}_{\delta,\nu}
(\mathbf{x}\mid\mathbf{x}^{-})$.
\end{proposition}

\begin{proof}
We proceed by structural induction over the nodes of the circuit,
establishing for every node $n$ that the transformed subcircuit computes
\[
\widetilde{p}_n(\mathbf{x}_{\phi(n)})
=\frac{p_n(\mathbf{x}_{\phi(n)})\exp\{-G_{\phi(n)}\}}{Z_n},
\qquad
Z_n=\int p_n\exp\{-G_{\phi(n)}\}\,d\mathbf{x}_{\phi(n)},
\]
the integral being a sum where the scope is discrete.

\paragraph{Input nodes.}
For $n$ with $\phi(n)=\{j\}$ we have $G_{\phi(n)}=g_j(x_j;x_j^{-})$, and the
claim is the definition of $\widetilde{p}_n$ together with
$Z_n=\int p_n(x_j)e^{-g_j}\,dx_j$, which is computable exactly by assumption.
Since $g_j\ge 0$ we have $e^{-g_j}\le 1$ and hence $0<Z_n\le 1$, so
$\widetilde{p}_n$ is a well-defined distribution over $\mathcal{X}_j$.

\paragraph{Product nodes.}
Let $n$ be a product node. By decomposability the children have disjoint
scopes with $\phi(n)=\bigcup_{c\in\operatorname{ch}(n)}\phi(c)$, and by
additive separability $G_{\phi(n)}=\sum_{c}G_{\phi(c)}$, so the tilt
factorizes: $\exp\{-G_{\phi(n)}\}=\prod_{c}\exp\{-G_{\phi(c)}\}$. With
$p_n=\prod_c p_c$, the integral factorizes:
\[
Z_n=\int\prod_c p_c\,e^{-G_{\phi(c)}}\,d\mathbf{x}_{\phi(n)}
   =\prod_c\int p_c\,e^{-G_{\phi(c)}}\,d\mathbf{x}_{\phi(c)}
   =\prod_c Z_c,
\]
and therefore
$\prod_c\widetilde{p}_c=\prod_c p_c e^{-G_{\phi(c)}}/Z_c
= p_n e^{-G_{\phi(n)}}/Z_n=\widetilde{p}_n$, as claimed.

\paragraph{Sum nodes.}
Let $n$ be a sum node. By smoothness all children share the scope
$\phi(n)$, so the same tilt factor $e^{-G_{\phi(n)}}$ multiplies each term of
$p_n=\sum_c\theta_{n,c}p_c$. Integrating,
\[
Z_n=\int\sum_c\theta_{n,c}\,p_c\,e^{-G_{\phi(n)}}\,d\mathbf{x}_{\phi(n)}
   =\sum_c\theta_{n,c}Z_c .
\]
With $\widetilde{\theta}_{n,c}=\theta_{n,c}Z_c/Z_n$ we obtain
\[
\sum_c\widetilde{\theta}_{n,c}\widetilde{p}_c
=\sum_c\frac{\theta_{n,c}Z_c}{Z_n}\cdot\frac{p_c e^{-G_{\phi(n)}}}{Z_c}
=\frac{e^{-G_{\phi(n)}}}{Z_n}\sum_c\theta_{n,c}p_c
=\widetilde{p}_n .
\]
Moreover $\sum_c\widetilde{\theta}_{n,c}
=\frac{1}{Z_n}\sum_c\theta_{n,c}Z_c=1$, so the transformed node is again a
normalized sum node.


\paragraph{Root.}
Applying the induction at the root $r$, where $\phi(r)=\{1,\dots,d\}$ and
$p_r=p^{+}$, gives $\widetilde{p}_r=p^{+}\exp\{-G\}/Z$, which is the TRD by definition.
\end{proof}

\subsection{Proof of Proposition 2}  \label{app:proof-conce}

\begin{proposition}[Concentration of Intervention Cost]
Let $\tilde p(\mathbf{x}\mid\mathbf{x}^-) = p^{+}(\mathbf{x})\exp\{-G(\mathbf{x};\mathbf{x}^-)\}/Z_G(\mathbf{x}^-)$,
where $G(\mathbf{x};\mathbf{x}^-)\ge 0$. Then, for every $\epsilon>0$,
\[
\Pr_{\mathbf{x}\sim\tilde p}\!\left[G(\mathbf{x};\mathbf{x}^-)>\epsilon\right]
\;\le\; \frac{\exp(-\epsilon)}{Z_G(\mathbf{x}^-)} .
\]

\end{proposition}

\begin{proof}
Fix $\epsilon>0$. Since $G\ge 0$, the event $\{G>\epsilon\}$ can be written as
$\{e^{G}>e^{\epsilon}\}$. Applying Markov's inequality to the non-negative
random variable $e^{G}$,
\[
\Pr[G>\epsilon] = \Pr[e^{G}>e^{\epsilon}]
\;\le\; \frac{\mathbb{E}_{\tilde p}[e^{G}]}{e^{\epsilon}} .
\]
Expanding the expectation,
\[
\mathbb{E}_{\tilde p}\!\left[e^{G}\right]
= \int e^{G(\mathbf{x};\mathbf{x}^-)}\,
  \frac{p^{+}(\mathbf{x})\,e^{-G(\mathbf{x};\mathbf{x}^-)}}{Z_G(\mathbf{x}^-)}\,d\mathbf{x}
= \frac{1}{Z_G(\mathbf{x}^-)}\int p^{+}(\mathbf{x})\,d\mathbf{x}
= \frac{1}{Z_G(\mathbf{x}^-)},
\]
since $p^{+}$ is a normalized density. Substituting gives the claim.
\end{proof}

\section{Hyperparameter Selection}                 \label{app:hparams}

\paragraph{Tilt weights.} The two tilt strengths $(\delta, \nu)$ are the only hyperparameters. Proximity and sparsity improve monotonically as either strength grows (Tables~\ref{tab:delta_sweep} and~\ref{tab:nu_sweep}), while coverage is unaffected until the strength becomes large enough that so much mass concentrates on the denied factual that too little remains on the accepted side of the boundary for any valid recourse to be sampled. Any strength below that collapse boundary is therefore dominated by a larger one, which gives closer and sparser recourses while serving every individual; we accordingly select the largest that succeeds.

Selection uses 30 denied factuals drawn from the test split, which is disjoint from the subtest split on which all results are reported. A factual succeeds at ($\delta, \nu$) if at least three feasible samples survive from $N=3000$ draws, three being the minimum from which the clustering step can form a recourse set. We sweep $\delta$ over $\{0.1, 0.25,0.5,1,2,3,4,6,8\}$ at $\nu=1$, then $\nu$ over $\{0,0.5,1,1.5,2,3\}$ at the chosen $\delta$, and take the minum over folds so that deployed strength succeeds on every fold. At deployment we draw $N=10{,}000$ samples per factual. The resulting operating points are given in Table~\ref{tab:lambdas}.

\paragraph{Tilt strength on MNIST.}
We study five source-to-target tasks: $8\rightarrow0$, $7\rightarrow1$, $8\rightarrow6$, $1\rightarrow4$, and $4\rightarrow1$. Images carry no immutability, monotonicity, or causal constraints, so feasibility reduces to validity: a sample is feasible if a binary source-versus-target classifier, trained per task, accepts it as the target digit. We draw $100$ factuals per task; test images of the source digit the classifier does not accept as the target, taking those with the lowest target-class probability. Each task has its own circuit, classifier and factuals, so the strength is selected per task.

Selection uses two criteria. As in the tabular benchmarks, we take the largest strength at which every tuning factual is served. That alone is too weak a requirement here: a factual can keep the three samples the rule demands while most of its samples are already invalid. So we require, in addition that mean validity stay above $80\%$. This gives $\nu=3$ for four tasks, and $\nu=4$ for $8\rightarrow0$; we report all five at $\nu=3$, at equal tilt strength. 


\begin{table}[!ht]
\centering\small
\begin{tabular}{lccc}
\toprule
 & Adult & Credit & GMSC \\
\midrule
Proximity $\delta$ & $0.6$ & $2.0$ & $2.0$ \\
Sparsity $\nu$     & $1.0$ & $1.5$ & $1.0$ \\
\bottomrule
\end{tabular}
\caption{Selected tilt strengths.}
\label{tab:lambdas}
\end{table}
\section{Sampling-Budget Ablation}
\label{app:budget}

The  configuration draws $N{=}10{,}000$ samples per individual. To check whether the returned recourse set depends on that choice, we sweep the budget over $M \in \{125, 250, \dots, 16{,}000\}$, holding the tilt, seed and filters fixed, and measure the diversity of both the feasible pool and the returned set.

Figure~\ref{fig:budget_ablation} shows that the two respond very differently. The feasible pool's diversity grows steadily with the budget: over a $128\times$ increase in $M$, the number of distinct changed-feature sets in the pool grows by $3.4\times$ on Adult, $15.2\times$ on German Credit and $2.0\times$ on GMSC. The returned set does not follow. Its $k_{\mathrm{eff}}$ is flat or slightly decreasing ($\times 0.86$, $\times 0.91$, $\times 0.84$ respectively), and Count-Diversity likewise declines ($\times 0.75$, $\times 0.81$, $\times 0.91$). Coverage saturates early: every factual is served by $M \approx 500$--$2000$ on all three datasets.
\begin{figure}[!ht]
    \centering
    \includegraphics[width=1\linewidth]{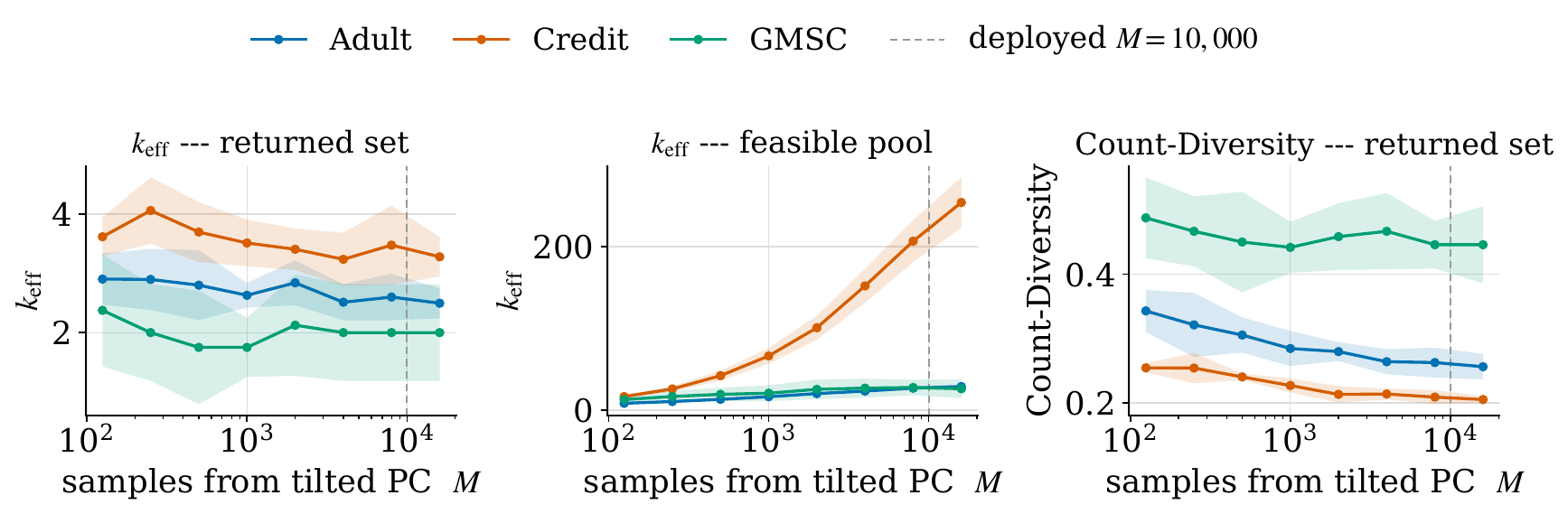}
    \caption{Sampling-budget sweep. Left: $k_{\mathrm{eff}}$ of the returned set.
Centre: $k_{\mathrm{eff}}$ of the feasible pool. Right: Count-Diversity of the
returned set. Bands are $\pm1$ standard deviation over folds; the dashed rule
marks the deployed budget. Pool diversity grows with the budget; returned
diversity does not.}
    \label{fig:budget_ablation}
\end{figure}

The reason is that $k_{\mathrm{eff}}$ is bounded above by the number of recourses returned, and the silhouette-selected $k$ does not respond to pool size; it remains within $3.7$--$4.2$ on Adult and $3.4$--$4.1$ German Credit and $2.5$--$2.8$ on GMSC throughout the sweep. Consequently, the fraction of the pool's distinct strategies that reach the user falls as the budget grows, from $34.3\%$ to $8.6\%$ on Adult, $21.7\%$ to $1.3\%$ on German Credit, and $18.3\%$ to $7.6\%$ on GMSC. The sampling budget is therefore not the binding constraint on the diversity of what we return; the number of clusters $k$ selected for the returned set is.

\section{Tilt-Strength Sensitivity}
\label{app:tilt-sensitivity}

Q4 in the main paper compares the two tilt terms at a single operating point. Here we vary each in turn. Table~\ref{tab:delta_sweep} sweeps $\delta$ at each dataset's deployed $\nu$, and Table~\ref{tab:nu_sweep} sweeps $\nu$ at its deployed $\delta$; the deployed value is marked $\dagger$ in each.

\begin{table*}[!ht]
\centering\small
\begin{minipage}[t]{0.48\textwidth}
\centering
\setlength{\tabcolsep}{5pt}
\begin{tabular}{@{}l c ccc@{}}
\toprule
Dataset & $\delta$ & Prox.\,$\downarrow$ & Spars.\,$\downarrow$ & Served\,\%\,$\uparrow$ \\
\midrule
\multirow{5}{*}{Adult}
 & 0.0 & 17.05 & 4.86 & 100 \\
 & 0.6$^{\dagger}$ & 13.40 & 4.07 & 100 \\
 & 2.0 & 10.17 & 2.71 & \phantom{0}74 \\
 & 4.0 & \phantom{0}9.73 & 2.12 & \phantom{0}58 \\
 & 8.0 & 11.38 & 1.72 & \phantom{0}35 \\
\midrule
\multirow{4}{*}{Credit}
 & 0.0 & 17.00 & 6.49 & 100 \\
 & 2.0$^{\dagger}$ & \phantom{0}9.83 & 4.01 & 100 \\
 & 4.0 & 10.50 & 4.18 & \phantom{0}96 \\
 & 8.0 & \phantom{0}9.93 & 4.14 & \phantom{0}74 \\
\midrule
\multirow{5}{*}{GMSC}
 & 0.0 & 20.40 & 7.67 & 100 \\
 & 2.0$^{\dagger}$ & 12.71 & 5.83 & 100 \\
 & 3.0 & 12.23 & 5.67 & 100 \\
 & 4.0 & 10.32 & 5.67 & 100 \\
 & 8.0 & \phantom{0}9.41 & 5.83 & 100 \\
\bottomrule
\end{tabular}
\caption{Proximity control: varying $\delta$ with $\nu$ held at its deployed
value, so the $\delta{=}0$ row is sparsity-only rather than untilted (cf.\
Figure~2, where \emph{neither} sets both strengths to zero).
$^{\dagger}$~deployed operating point.}
\label{tab:delta_sweep}
\end{minipage}\hfill
\begin{minipage}[t]{0.48\textwidth}
\centering
\setlength{\tabcolsep}{5pt}
\begin{tabular}{@{}l c ccc@{}}
\toprule
Dataset & $\nu$ & Spars.\,$\downarrow$ & Prox.\,$\downarrow$ & Served\,\%\,$\uparrow$ \\
\midrule
\multirow{5}{*}{Adult}
 & 0.0 & 4.48 & 14.61 & 100.0 \\
 & 1.0$^{\dagger}$ & 4.07 & 13.40 & 100.0 \\
 & 2.0 & 3.69 & 12.95 & \phantom{0}99.5 \\
 & 3.0 & 3.35 & 11.49 & \phantom{0}96.5 \\
 & 4.0 & 3.07 & 10.47 & \phantom{0}84.5 \\
\midrule
\multirow{5}{*}{Credit}
 & 0.0 & 5.35 & 13.93 & 100.0 \\
 & 1.5$^{\dagger}$ & 4.01 & \phantom{0}9.83 & 100.0 \\
 & 2.0 & 3.88 & \phantom{0}9.16 & \phantom{0}98.6 \\
 & 3.0 & 3.66 & \phantom{0}8.51 & \phantom{0}97.2 \\
 & 4.0 & 3.47 & \phantom{0}7.64 & \phantom{0}88.7 \\
\midrule
\multirow{5}{*}{GMSC}
 & 0.0 & 6.17 & 13.67 & 100.0 \\
 & 1.0$^{\dagger}$ & 5.83 & 12.71 & 100.0 \\
 & 2.0 & 5.83 & 11.16 & 100.0 \\
 & 3.0 & 5.50 & 11.28 & 100.0 \\
 & 4.0 & 4.83 & 11.07 & 100.0 \\
\bottomrule
\end{tabular}
\caption{Sparsity control: varying $\nu$ with $\delta$ held at its deployed
value, so the $\nu{=}0$ row is proximity-only rather than untilted.
$^{\dagger}$~deployed operating point.}
\label{tab:nu_sweep}
\end{minipage}
\end{table*}

Both tilts behave monotonically in the direction they target: increasing $\delta$ reduces proximity throughout, and increasing $\nu$ reduces sparsity throughout. Each also improves the other metric, which is the complementarity reported in Q4. The cost appears in coverage. On Adult, raising $\delta$ from its deployed $0.6$ to $4.0$ cuts distance from $13.40$ to $9.73$ but drops Served \% from $100$ to $58$; at $\delta{=}8$ only $35\%$ of factuals are served. German Credit shows the same pattern more gently, reaching $96\%$ at $\delta=4$ and $74\%$ at $\delta=8$. On GMSC coverage stays at $100\%$ even at $\delta=8$, consistent with its higher feasible mass, and the cost surfaces instead as plausibility: the best-member NLL rises from $5.46$ at $\delta=0$ to $12.67$ at $\delta=8$. The selection rule of Section~\ref{app:hparams} takes the largest strength preserving full coverage on the tuning split, where the collapse boundary is reached. The denied subtest individuals on GMSC happen to lie well inside it.


\section{Tilted Input Distributions}
\label{app:leaves}

Corollary 1 requires that each input node's tilted distribution and its normalizer be computable exactly. Two facts make this mild. First, for any non-negative cost the tilt factor satisfies $e^{-g_j}\le 1$, so
\[
0 \;<\; Z_n \;=\; \int f_n(x_j)\,e^{-g_j(x_j;x_j^-)}\,dx_j \;\le\; 1
\]
at every input node: the tilted circuit always exists, whatever the feature type. Second, $Z_n$ is available in closed form for the standard input families. For a Gaussian input $\mathcal{N}(\mu,\sigma^2)$ and the $\ell_1$ proximity cost, splitting the integral at $c=x_j^-$ and completing the square gives
\[
Z_n = e^{\delta^2\sigma^2/2}\Big[
 e^{\delta(\mu-c)}\Phi\big(\tfrac{c-\mu-\delta\sigma^2}{\sigma}\big)
+e^{\delta(c-\mu)}\Phi\big(\tfrac{\mu-\delta\sigma^2-c}{\sigma}\big)\Big],
\]
a tilted input is a two-piece truncated Gaussian with means $\mu\mp\delta\sigma^2$ either side of $c$. The same construction applies to Gamma, exponential, Poisson and uniform inputs; more generally it holds whenever the input family's moment generating function and CDF are available in closed form, since $|x_j-x_j^-|$ is piecewise linear and the tilt reduces to a natural-parameter shift on either side of the factual.

 Under a continuous input, $\Pr(x_j = x_j^-)=0$, so $\mathbb{1}[x_j\neq x_j^-]=1$ almost surely and $e^{-\nu\mathbb{1}[\cdot]}$ contributes a constant that cancels in the normalization: only the proximity tilt acts.

\section{Datasets, Constraints and, Implementation} 
\label{app:setup}

We use the LiCE benchmark's constraint specifications unmodified. Adult declares one causal implication (education $\Rightarrow$ age), German Credit two (residence $\Rightarrow$ age, employment $\Rightarrow$ age), and GMSC none; the Causal \% column of Table~2 is therefore omitted for GMSC. Immutable and monotone attributes follow the benchmark's definitions.

\paragraph{Computing infrastructure.} All experiments were run on a single Apple MacBook Air (M3, 16\,GB RAM) under macOS~14.5, on CPU only; no GPU was
used. The implementation uses Python~3.9.6 with NumPy~2.0.2, pandas~2.3.3, SciPy~1.13.1, scikit-learn~1.6.1 and PyTorch~2.8.0 (CPU build). Probabilistic circuits are learned and evaluated with SPFlow, and the classifiers, data splits and constraint specifications are taken unmodified from the LiCE benchmark. SPFlow is included in our code release; the LiCE benchmark is not, as its repository carries no licence permitting redistribution. Generating recourse for one denied individual takes under one second at a sampling budget of $N=10,000$ on this hardware. All runs use seed~0.

\section{Filter Attribution}
\label{app:filters}

Actionability is enforced by construction, so the only samples lost are those failing classifier validity or the causal constraints. Table~\ref{tab:filters} attributes the loss to each stage.

Validity accounts for almost all of it: $82.6\%$ of discarded samples on Adult, $98.0\%$ on German Credit, and all of it on GMSC, which declares no causal implications. Among samples that pass validity, $56.5\%$ on Adult and $89.3\%$ on German Credit also satisfy the causal constraints.

The resulting pools are large. Median feasible counts per individual are $1351$ on Adult, $1262$ on German Credit and $8363$ on GMSC, against the three samples the clustering step requires.

\begin{table}[!ht]\centering\small
\begin{tabular}{@{}lccccc@{}}
\toprule
 & & & \multicolumn{2}{c}{loss attributed to} & median \\
\cmidrule(lr){4-5}
Dataset & Valid \% & Feasible \% & validity & causal & feasible \\
\midrule
Adult         & 32.7 & 18.5 & 82.6\% & 17.4\% & 1351 \\
German Credit & 19.7 & 18.1 & 98.0\% & 2.0\%  & 1262 \\
GMSC          & 71.1 & 71.1 & 100\%  & ---    & 8363 \\
\bottomrule
\end{tabular}
\caption{Per-factual acceptance rates and attribution of rejected samples.}
\label{tab:filters}
\end{table}
\section{Additional MNIST Results}  
\label{app:mnist}
Q5 in the main paper shows two of the five source-to-target tasks. The remaining three are reported here in the same format. Figure~\ref{fig:mnist_supp} shows samples and their pixel differences from the factual at $\nu\in\{0,3,7\}$, and Table~\ref{tab:mnist_supp} gives recourse quality at the deployed $\nu=3$. The behavior matches the two main tasks.

\label{app:mnist}
\begin{figure}[!ht]
    \centering
    \includegraphics[width=1\linewidth]{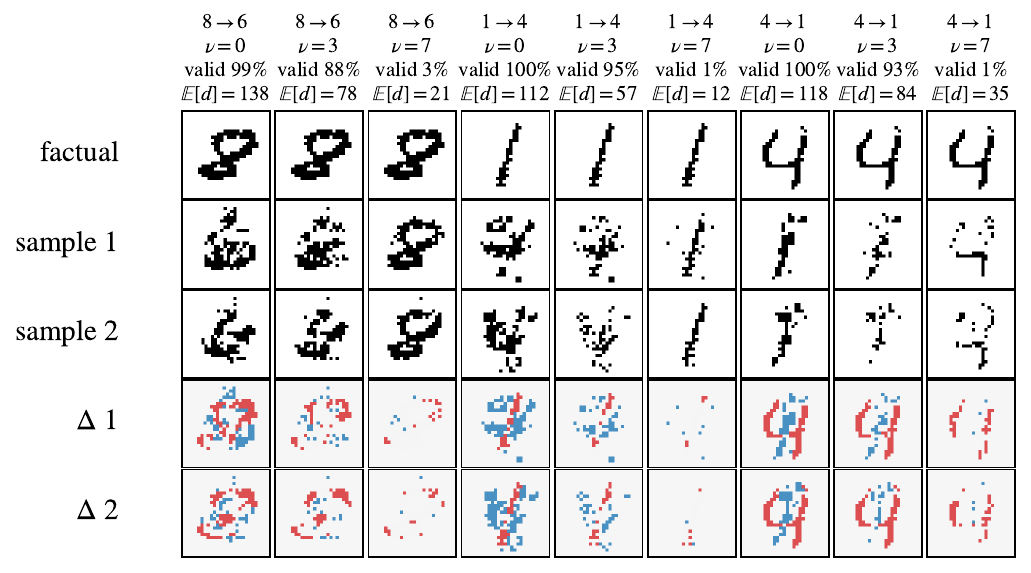}
    \caption{The tilt on MNIST. Each row tilts a target-digit circuit toward a factual source-digit image (left) at strength $\nu$; \emph{sample 1 - 2} are independent draws from the tilted circuit and $\Delta 1$--$2$ their pixel differences from the factual (red: present in the factual and removed by the sample; blue: added by the sample).}
    \label{fig:mnist_supp}
\end{figure}
\begin{table}[!ht]
\centering\small
\setlength{\tabcolsep}{4pt}
\begin{tabular}{@{}l c ccc cc@{}}
\toprule
& & \multicolumn{3}{c}{NLL\,$\downarrow$} & & \\
\cmidrule(lr){3-5}
Task & Valid\,\%\,$\uparrow$ & best & mean & worst
& Spars.\,(px)\,$\downarrow$ & Count-Div\,$\uparrow$ \\
\midrule
$8\rightarrow6$ & 83.3 & 110.5 & 160.7 & 222.2 & 80.5 & 0.104 \\
$1\rightarrow4$ & 92.0 & 90.8 & 144.4 & 223.2 & 60.7 & 0.095 \\
$4\rightarrow1$ & 88.0 & 50.8 & 94.6 & 152.2 & 90.2 & 0.063 \\
\bottomrule
\end{tabular}
\caption{MNIST recourse quality at the operating point $\nu{=}3$
(all factuals served).}
\label{tab:mnist_supp}
\end{table}

\section{Proximity and Sparsity Comparison}
\label{app:prox_spars}

Proximity and sparsity are not objectives TRD optimizes directly: the tilt reweights the recourse distribution so that closer, sparser candidates carry more probability mass. LiCE and MIO minimize the distance explicitly within a mixed-integer program, and Table~\ref{tab:prox_spars} shows the expected consequence: they attain the lowest proximity and sparsity on all three datasets, with TRD between them and DiCE.
\begin{table*}[!ht]\centering\small
\setlength{\tabcolsep}{4pt}
\begin{tabular}{@{}l cccc cccc cccc@{}}
\toprule
& \multicolumn{4}{c}{Adult} & \multicolumn{4}{c}{German Credit}
& \multicolumn{4}{c@{}}{GMSC} \\
\cmidrule(lr){2-5}\cmidrule(lr){6-9}\cmidrule(l){10-13}
& \multicolumn{2}{c}{best member} & \multicolumn{2}{c}{whole set}
& \multicolumn{2}{c}{best member} & \multicolumn{2}{c}{whole set}
& \multicolumn{2}{c}{best member} & \multicolumn{2}{c@{}}{whole set} \\
\cmidrule(lr){2-3}\cmidrule(lr){4-5}\cmidrule(lr){6-7}\cmidrule(lr){8-9}\cmidrule(lr){10-11}\cmidrule(l){12-13}
Method & Prox. & Spars. & Prox. & Spars.
       & Prox. & Spars. & Prox. & Spars.
       & Prox. & Spars. & Prox. & Spars. \\
\midrule
DiCE & 28.25 & 4.11 & 37.87 & 4.79
     & 14.19 & 5.00 & 17.08 & 6.00
     & 14.78 & 6.75 & 13.99 & 6.08 \\
LiCE & 10.36 & 3.22 & 10.36 & 3.22
     & \phantom{0}2.69 & \textbf{1.50} & \textbf{\phantom{0}2.68} & \textbf{1.52}
     & \phantom{0}6.06 & 4.12 & \phantom{0}6.06 & 4.12 \\
MIO  & \textbf{\phantom{0}6.46} & \textbf{2.50} & \textbf{\phantom{0}6.49} & \textbf{2.75}
     & \textbf{\phantom{0}2.68} & 1.91 & \phantom{0}2.70 & 2.60
     & \textbf{\phantom{0}3.61} & \textbf{2.88} & \textbf{\phantom{0}3.57} & \textbf{2.98} \\
\textbf{TRD} & 13.58 & 4.09 & 14.50 & 4.08
     & \phantom{0}9.70 & 4.03 & \phantom{0}9.19 & 3.77
     & 13.77 & 5.62 & 12.16 & 5.81 \\
\bottomrule
\end{tabular}
\caption{Proximity (MAD-weighted $\ell_1$, $\downarrow$) and sparsity
($\downarrow$) of returned recourse sets; means over folds, on each method's
feasible subset. \emph{Best member} is the min-NLL counterfactual;
\emph{whole set} averages over all returned. LiCE is the median variant.
\textbf{Bold}: best per column within each dataset.}
\label{tab:prox_spars}
\end{table*}

The columns understate TRD in two respects. First, among the methods that target diversity, TRD improves on DiCE in every column and on every dataset. Second, the averages are taken over sets of different composition: LiCE  returns ten counterfactuals that collapse to a single distinct strategy (cf.\ Table~1), so its whole-set columns are averages over near-identical points and coincide with its best-member values, whereas DiCE and TRD average over several distinct strategies each. TRD therefore trades some proximity for the plausibility and diversity reported in Table~1, where it attains the lowest best and mean NLL on all three datasets.

\end{document}